\documentclass[12pt]{article}
\usepackage{amsmath, amsthm, amssymb}
\usepackage{graphicx}
\usepackage{hyperref}
\usepackage{comment}
\usepackage{appendix}
\usepackage{booktabs}

\usepackage{comment}
\usepackage{blkarray}
\usepackage{algorithm}
\usepackage{algpseudocode}
\usepackage[table]{xcolor} 
\usepackage{listings}
\usepackage{booktabs}
\usepackage{enumitem}
\usepackage[english]{babel}
\usepackage{pgfplots}    
\pgfplotsset{compat=newest}

\newtheorem{theorem}{Theorem}[section]
\newtheorem{lemma}[theorem]{Lemma}
\newtheorem{corollary}[theorem]{Corollary}
\theoremstyle{definition}
\newtheorem{definition}[theorem]{Definition}
\newtheorem{assumption}{Assumption}
\newtheorem{proposition}[theorem]{Proposition}

\numberwithin{theorem}{section}     

\usepackage{parskip}

\newcommand{\sgn}{\operatorname{sgn}}
\newcommand{\MSE}{\operatorname{MSE}}

\newcommand{\textsub}[1]{\ensuremath{_{\textsc{#1}}}}

\newcommand{\QSZT}{Q\textsub{szt}}
\newcommand{\QBT}{Q\textsub{bt}}
\newcommand{\HSZT}{H\textsub{szt}}
\newcommand{\HBT}{H\textsub{bt}}
\newcommand{\tauSZT}{\tau\textsub{szt}}
\newcommand{\tauBT}{\tau\textsub{bt}}

\newcommand{\zpos}{0^{+}}
\newcommand{\zneg}{0^{-}}

\newcommand{\PhiF}{\Phi_{\!F}}
\newcommand{\PhiR}{\Phi_{\!R}}

\begin{document}

\title{The Fourth State: Signed-Zero Ternary for\\ 
         Stable LLM Quantization (and More)}
\author{Jeffrey Uhlmann\\
Dept.\ of Electrical Engineering and Computer Science\\
University of Missouri - Columbia}
\date{}
\maketitle

\begin{abstract}
Quantization is usually regarded as a means to trade quality of performance for reduced compute
requirements, i.e., as a suboptimal approximation. However, if examined in terms of a fixed
overall resource budget, a very different perspective arises. We introduce Signed-Zero Ternary (SZT),
a 2-bit quantization that deterministically provides gradient information with no forward-path 
penalty. Our analysis provides evidence that it may improve information density compared to 
non-quantized alternatives.\\ 
\begin{footnotesize}
~ 
\end{footnotesize}
\end{abstract}

\section{Introduction}\label{sec:intro}

Quantization is typically viewed as a pragmatic trade-off between model fidelity 
and computational costs \cite{courbariaux2016dorefa,banner2018post,dettmers2022eightbit}. 
Aggressive $2$-bit ternary-state schemes are now commonly used to allow large-language 
models (LLMs) to run on commodity accelerators and edge devices. However, this leads to 
training-time issues resulting from intervals in which the quantizer output is numerically
zero and the surrogate gradient vanishes. These near-zero intevals are referred
to as ``dead zones'' \cite{sehwag2020bnptq}.  

We introduce a Signed-Zero Ternary (SZT) quantization in which we use the 
remaining fourth state in the 2-bit ternary encoding to distinguish {\em two}
zero states (code words). This approach retains the benefits of ternary-state
quantization while adding 1-bit gradient information at essentially no cost.
This preserves the forward-path behavior of balanced ternary while the
back-propagation rule remains fully deterministic for the straight-through form.
We argue that availability of gradient information in this maximally quantized
representation may tend to maximize overall information density rather than
approximate it. All analytical results are obtained via changes to the 
encode/decode logic only, leaving the matrix-multiply datapath 
untouched, i.e., no kernel changes.

\subsection{Motivation: why a fourth state matters}\label{sec:intro-motivation}

Balanced ternary quantization \cite{zhu2017ttq} encodes $\{-1,0,+1\}$ using two bits, which 
leaves one of the four available states unused.  A consequence of having only three states is 
that during quantization-aware training (QAT), a weight falling in the dead-zone interval 
$|w|\leq\Delta$ produces the same quantized value for arbitrarily many SGD steps because 
the optimizer receives no feedback until the update magnitude exceeds the threshold $\Delta$.

Empirical histograms of pretrained LLM checkpoints \cite{frantar2022gptq} 
show sustantial mass near zero \cite{sehwag2020bnptq}.
The SZT signed-zero states ($\zpos$ or $\zneg$) provide the optimizer with a gradient for 
every sub-threshold sign flip while leaving the forward alphabet unchanged.
Unlike approaches that use randomization to stochastically push weights 
out of the dead zone, SZT's extra bit of feedback provides a purely deterministic 
solution that we will show can significantly reduce the expected mean first-passage 
time out of the dead zone.

Dead-zone instability has historically made designers favor smaller, higher-precision models. 
However, by addressing the dead-zone problem deterministically, SZT represents an
alternative that supports twice the number of parameters of a 4-bit model within the 
same memory footprint.

\subsection{Related work}
\label{sec:intro-related}

Classical post-training quantization (PTQ) applies a fixed mapping to
full-precision weights and activations after training
\cite{jacob2018qtensor,nagel2020aas}.  Quantization-aware training extends
this idea by inserting surrogate gradients, e.g., the straight-through
estimator (STE) \cite{bengio2013estimating}, into the backward graph, allowing
the optimizer to compensate for quantization noise.  Several authors propose
stochastic rounding in quantization
\cite{zhang2020bitrandom}, but at the cost of non-determinism
and additional random-number generation hardware.  

As already mentioned, an important distinction of our proposed SZT approach
is that it is {\em deterministic}: both the encoder and the surrogate
gradient are fixed functions of the latent weight, which guarantees reproducible
training trajectories (Section~\ref{sec:determinism}). This deterministic property
also makes it more amenable to sensitivity, entropy, and bias analysis.

The structure of this paper is as follows:
Section~\ref{sec:szt} formalizes the SZT quantizer and threshold rules.
Section~\ref{sec:advantages} presents the analytical results, with proofs
deferred to the appendices.  And Section~\ref{sec:discussion} discusses hardware
implications and open empirical questions for future work.

\paragraph{Scope of analysis.}
Throughout the paper we adopt the {\em symmetric, layer-wise, single-scale}
quantization model that underlies most open-source PTQ and QAT tool-chains
(e.g., GPTQ, AWQ, SmoothQuant).  However, all analytic results presented here (sensitivity ratio,
entropy gap, bias/variance, MFPT, PAC–Bayes) generalize 
channel-wise by replacing the single threshold $\Delta$ with a
channel-specific $\Delta_{c}$. Our proofs can then be extended
at the cost of significant extra bookkeeping and an inability in some
cases to derive ``nice'' closed-form constants. Section~\ref{sec:tradeoff} 
and Appendix~\ref{app:snr} 
outline straightforward extensions to the case of per-channel 
or asymmetric thresholds.

\section{Balanced Ternary (BT) vs.\ Signed-Zero Ternary (SZT)}
\label{sec:szt}

\subsection{Dead-zone pathology}
\label{sec:deadzone}

Balanced-ternary (BT) quantization encodes a latent weight
$w\in\mathbb{R}$ using the three\-symbol alphabet
$\{-1,\,0,\,+1\}$ and a single symmetric threshold
$\Delta>0$:
\begin{equation}
\QBT(w) ~=~
\begin{cases}
     +1,       & ~w ~>~  \Delta,\\
     ~~0,  & |w| \leq~ \Delta,\\
     -1,        & ~w ~< -\Delta.
\end{cases}
\label{eq:bt}
\end{equation}
As noted earlier this leaves one of the four available code-words
as an unused {\em idle} or {\em don't-care} state. This isn't an issue 
during the forward pass, but during training 
the dead-zone  interval $[-\,\Delta,\,\Delta]$ implies that the true 
Jacobian $\frac{\partial \QBT}{\partial w}$ is zero almost 
everywhere within that interval.

We denote the update magnitude for a single step by
$s\triangleq\alpha\lVert g\rVert$ (i.e., learning rate $\alpha$, gradient
$g$).  For sub-threshold steps $0<s<\Delta$,
the probability that a BT-quantized weight changes its state is
\begin{equation}
     \PhiF(s) ~=~ 2\int_{\Delta-s}^{\Delta} p(w)~dw,
\end{equation}
where $p(w)$ is the latent weight density.
Because $p(w)$ is typically small near $\pm\Delta$,
$\PhiF(s)$ is strongly suppressed (cf., Section~\ref{sec:adv-sens}).
This results in a ``flattening'' of the loss landscape around zero that tends
to reduce convergence and increase the mean first-passage time (MFPT)
out of the dead zone. In the next section we formally define the 
Signed-Zero Ternary (SZT) approach for mitigating these effects.

\subsection{Formal Definition of the Quantizer}
\label{subsec:sztDefinition}

Let $w \in \mathbb{R}$ be a latent (full-precision) weight and let $\Delta>0$ denote a fixed quantization threshold. The four-state quantizer $\QSZT(w)$ is defined as:
\begin{equation}
\QSZT(w) := 
\begin{cases}
+1,         & ~~~~~ w > \Delta, \\
~\zpos, & ~~~ 0 < w \leq \Delta, \\
~\zneg,  & -\Delta \leq w < 0, \\
-1,          & ~~~~~~w < -\Delta,
\end{cases}
\label{eq:sztQuantizer}
\end{equation}
where $\zpos$ and $\zneg$ are distinct states that share the same numeric value in the forward pass. 
To acommodate this, we introduce a decode function $v(q)$:
\begin{equation}
v(q) ~:=~ 
\begin{cases}
     +1 & \text{if } q = +1,\\
     ~~0  & \text{if } q \in \{\zpos, \zneg\},\\
     -1 & \text{if } q = -1,
\end{cases}
\label{eq:decodeFunction}
\end{equation}
so that matrix multiplication kernels can operate on the conventional ternary alphabet $\{-1,0,+1\}$ with no change to data-path arithmetic, i.e., only encode/decode logic changes.

\subsection{Deterministic straight-through estimator (STE)}

During back-propagation, a surrogate Jacobian is needed for the non-differentiable
quantizer.  Following the straight-through paradigm \cite{bengio2013estimating},
we propagate the upstream gradient $\nabla_{q}\mathcal{L}$ unchanged outside the
dead zone but multiply by the stored sign when inside that interval:
\begin{equation}
\label{eq:sztSte}
     \widehat{\nabla}_w\mathcal{L} ~:=~
\begin{cases}
     \nabla_q\mathcal{L}, & |w|>\Delta,\\
     \sgn(q) \nabla_q \mathcal{L}, & |w|\leq\Delta,
\end{cases}
\end{equation}
where: 
\begin{equation}
     \sgn(q) ~=~ 
\begin{cases}
     +1 & \text{if}~ q\in\{\zpos,+1\}\\
     -1  & \text{if}~ q\in\{\zneg,-1\}
\end{cases}
\end{equation}
Because this rule is fully deterministic, identical
inputs and random seed reproduce {\em bit-exact} training trajectories
(Lemma~\ref{lem:repro}).

\subsection{Threshold selection for SZT}
\label{sec:threshold}

SZT's only free parameter is the symmetric threshold $\Delta$. We propose the practical 
default of setting $\Delta$ equal to the standard deviation of the latent weights ($\sigma$) 
on a per-layer basis:
\begin{equation}
\label{eq:delta-equals-sigma}
    \Delta ~=~ \sigma
\end{equation}
This default choice eliminates the only tunable parameter, and as detailed in 
Appendices~\ref{app:delta-laplace} and~\ref{app:threshold}, it is provably 
optimal for a Laplace prior; near-optimal for a Gaussian prior; and increases the 
Mean-Squared Error (MSE) by less than 5\% from the theoretical minimum. 

All analytical results that follow are stated in terms of the dimensionless ratio 
$k=\Delta/\sigma$ to allow for the direct substitution of alternative threshold 
heuristics.

\subsection{Implementation note: encode / decode only, GEMM unchanged}
\label{sec:impl}

As already discussed, Signed-Zero Ternary (SZT) is a {\em representation-level}
change: the numeric stream consumed by GEMM remains $\{-1,0,+1\}$, so every
matrix-multiply and sparsity kernel is {\em byte-for-byte identical} to those
used for balanced ternary.

Weights are still packed two bits per value, i.e., the extra state only changes the
{\em meaning} of the idle code-word. No firmware patch is needed for edge
inference engines that already understand balanced-ternary blobs.

In summary, hardware that currently runs balanced ternary 
can run SZT {\em compile-time only}. Below is one possible
formulation:
\begin{algorithm}[t]
\label{alg:qat-szt}
  \caption{Quantization-Aware Training (QAT) with Signed-Zero Ternary.
        (All analysis in Section~\ref{sec:advantages} assumes $\Delta$ is held fixed, but it can be
        after each epoch or phase without altering the deterministic update rule.)}
  \small
  \begin{algorithmic}[1]
    \Require
      training data $\{(x_i,y_i)\}_{i=1}^{N}$,
      initial weights $w_0$,
      threshold $\Delta$ (per layer),
      optimizer $\mathsf{Opt}$ with learning-rate schedule $\{\alpha_t\}$
    \Statex
    \For{$t = 0,1,2,\dots$}
      \State Sample mini-batch $\mathcal{B}_t\subset\{1,\dots,N\}$
      \State {\bf Encode}: $q_t \leftarrow \QSZT(w_t)$   \Comment{Alg.~\ref{eq:sztQuantizer}}
      \State {\bf Decode}: $\tilde{w}_t \leftarrow v(q_t)$           \Comment{Eq.~\ref{eq:decodeFunction}}
      \State Forward pass: $\hat{y} \leftarrow f(\tilde{w}_t, x_{\mathcal{B}_t})$
      \State Compute loss $\mathcal{L}_t \leftarrow \ell(\hat{y}, y_{\mathcal{B}_t})$
      \State Back-propagate with deterministic STE (Eq.~\ref{eq:sztSte})
             to obtain $\widehat{\nabla}_{w}\mathcal{L}_t$
      \State {\bf Update}: $w_{t+1} \leftarrow
              \mathsf{Opt}\left(w_t, \widehat{\nabla}_{w}\mathcal{L}_t, \alpha_t\right)$
    \EndFor
  \end{algorithmic}
\end{algorithm}


\section{Analytical Advantages of SZT}
\label{sec:advantages}

\begin{assumption}[Fixed quantization step]
\label{ass:fixed-delta}
Each layer uses a constant symmetric threshold $\Delta>0$ chosen once
during calibration (cf., Section~\ref{sec:threshold}).  Channel-wise thresholds
$\{\Delta_{c}\}$ yield identical bounds after substituting
$\Delta\rightarrow\Delta_{c}$.
\end{assumption}

\begin{assumption}[Symmetric unimodal prior]
\label{ass:symmetric}
Every layer’s pre-quantized weights follow an absolutely continuous,
{\em symmetric} and {\em unimodal} density $p(w)$ with $p(0)>0$ and
finite variance $\sigma^{2}=\mathbb{E}[w^{2}]$.
\end{assumption}

\begin{assumption}[Lipschitz loss]
\label{ass:lipschitz}
The loss $\mathcal{L}$ is $L$-Lipschitz in each weight.  Without loss of
generality we rescale so $L\leq 1$, i.e., this constant shows up in bias and
convergence bounds such as Theorem~\ref{thm:bias-var}.
\end{assumption}

\begin{assumption}[Ornstein–Uhlenbeck proxy]
\label{ass:ou-proxy}
For the mean-first-passage and convergence analysis in
Section~\ref{sec:adv-mfpt} we approximate the latent weight dynamics by the
Ornstein–Uhlenbeck SDE
\begin{equation}
     dW_t ~=~ -\kappa\,W_t\,dt + \sigma\,dB_t,
\end{equation}
where $\kappa>0$ (mean-reversion) and $\sigma>0$ (noise scale) are
layer-dependent constants and $B_t$ is standard Brownian motion.  This
assumption is {\em only} invoked by results that explicitly cite
Assumption~\ref{ass:ou-proxy}. All other sections rely only on
Assumptions~\ref{ass:fixed-delta}, \ref{ass:symmetric}, and \ref{ass:lipschitz}.
\end{assumption}

\begin{corollary}[Per-channel extension]
\label{cor:per-channel}
     All results in Sections~\ref{sec:adv-sens}–\ref{sec:adv-mfpt} remain valid under
     channel-wise thresholds $\{\Delta_c\}$ by applying each statement per channel after 
     the substitutions $\Delta\to\Delta_c$ and $p\to p_c$. Aggregating over channels gives:
     \begin{eqnarray}
          E_F &=& \sum_c N\,\mathbb{E}[\Phi_{F,c}(S)], \\
          E_R &=& \sum_c N\,\mathbb{E}[\Phi_{R,c}(S)],
     \end{eqnarray}
     and the distribution-free ratio bound becomes:
     \begin{equation}
          \frac{E_R}{E_F}\ \geq \min_c \frac{p_c(0)}{p_c(\Delta_c)}.
     \end{equation}
     The forward-SNR neutrality (Appendix~\ref{app:snr}) is unchanged because decoding 
     remains $\{-1,0,+1\}$ per channel.
\end{corollary}

In this section we give various analytic results while deferring all proofs
to the appendices.  We begin by quantifying how likely it is that a {\em tiny}
update will trigger a quantized state transition.

\subsection{Sensitivity analysis: representational vs.\ functional}
\label{sec:adv-sens}

\paragraph{Set-up.}
Consider a single weight $w$ whose latent update at step $t$ is
\begin{equation}
     w_{t+1} = w_{t}-s_{t}, ~~~ s_{t}\equiv\alpha_{t}g_{t},
\quad
0<s_{t}<\Delta,
\end{equation}
where $\alpha_{t}$ is the learning rate and $g_{t}$ the stochastic
gradient.  Let $p(w)$ denote the (symmetric, unimodal) density of $w$
immediately {\em before} the update.

\begin{definition}[Per-step sensitivities]\label{def:sensitivities}
\begin{eqnarray}
\PhiF(s) &=& \Pr\left(\text{$Q$ changes {\em numeric} value}\mid s\right) \\
                &=& 2\int_{\Delta-s}^{\Delta} p(w)~dw~\PhiR(s) \\
                &=& \Pr\left(\text{$Q$ changes {\em representation} only}\mid s\right) \\
                &=& 2\int_{0}^{s} p(w)~dw .
\end{eqnarray}
\end{definition}
$\Phi_F$ captures $\pm 1\leftrightarrow 0$ transitions (forward-path
changes), whereas $\Phi_R$ captures $\zpos\leftrightarrow \zneg$
sign flips unique to SZT, where $\Phi_R\equiv 0$ for balanced ternary.

\paragraph{Sensitivity gap.}
\begin{proposition}[Distribution-independent ratio]
\label{prop:ratio}
For any symmetric, unimodal density $p(w)$ and any sub-threshold step
$0<s<\Delta$:
\begin{equation}
\label{eq:ratio-main}
     \frac{\Phi_R(s)}{\Phi_F(s)} ~\geq~ \frac{p(0)}{p(\Delta)} ~>~ 1.
\end{equation}
\end{proposition}
Because empirical weight histograms typically satisfy $p(0)\gg p(\Delta)$
\cite{frantar2022gptq}, the actual ratio will likely be substantially
larger.  A Laplace prior with the MSE-optimal threshold
$\Delta=\sigma$ gives the closed form:
\begin{equation}
\label{eq:laplace-ratio}
     \frac{\Phi_R(s)}{\Phi_F(s)} ~=~ e^{\Delta/b}\,
     \frac{1 - e^{-s/b}}{e^{s/b} - 1}.
\end{equation}

\paragraph{Averaging over step sizes.}
Assume the small-step regime $S\in(0,\Delta)$ almost surely. For a Laplace prior
$p(w)=\tfrac{1}{2b}e^{-|w|/b}$:
\begin{eqnarray}
     \PhiR(s) &=& 1-e^{-s/b}, ~\text{and}  \\
     \Phi_F(s) &=& e^{-\Delta/b}\big(e^{s/b}-1\big), 
\end{eqnarray}
so integrating over $S\sim f$ gives:
\begin{eqnarray}
\label{eq:ratio-mgf}
     \frac{\mathbb{E}[\Phi_R(S)]}{\mathbb{E}[\Phi_F(S)]}
     &=& e^{\Delta/b}\,\frac{1-\mathbb{E}\!\left[e^{-S/b}\right]}{\mathbb{E}\!\left[e^{S/b}\right]-1} \\
     &=& e^{\Delta/b}\,\frac{1-\mathcal{M}_S(-1/b)}{\mathcal{M}_S(1/b)-1},
\end{eqnarray}
where $\mathcal{M}_S(\theta)=\mathbb{E}[e^{\theta S}]$ is the MGF of $S$.
For a deterministic step $S\equiv s_0$, this reduces to the closed form
$e^{\Delta/b}\frac{1-e^{-s_0/b}}{e^{s_0/b}-1}$.

We can also examine the expected feedback events per epoch.
Let an {\em epoch} consist of $N$ SGD steps per parameter, and let
$S\sim f(s)$ denote the random step magnitude distribution observed in
practice (typically heavy-tailed with $\mathbb{E}[S]\ll\Delta$).
Now let:
\begin{eqnarray}
     E_F &=& N\,\mathbb{E}\left[\PhiF(S)\right], \\
     E_R &=& N\,\mathbb{E}\left[\PhiR(S)\right],
\end{eqnarray}
be the expected number of numeric versus representational transitions, respectively, 
that a single weight experiences in one epoch.

\begin{corollary}[Feedback multiplier]
\label{cor:events}
Under the same conditions as Proposition~\ref{prop:ratio},
\begin{equation}
     \frac{E_R}{E_F} ~\geq~ \frac{p(0)}{p(\Delta)}.
\end{equation}
For transformer layers with $p(0)/p(\Delta) \approx 30$
\cite{frantar2022gptq}, a typical weight delivers
$30\times$ more sign transitions that provide state 
information to the optimizer per epoch under SZT
than under balanced ternary even though they have identical 
forward-path alphabets.
\end{corollary}

The proofs of Proposition~\ref{prop:ratio} and
Corollary~\ref{cor:events} appear only in
Appendix~\ref{app:sens-entropy}, i.e., no distributional assumption
beyond symmetry and unimodality is required.  

In summary, a relatively low frequency of deterministic SZT-enabled sign flips
is sufficient to keep optimizer momentum ``alive'' inside the dead zone.
We now proceed to examine the downstream consequences in terms of
reduced STE bias and shorter mean first-passage time relative to
balanced ternary.

\subsection{Information-theoretic advantage: the entropy gap}
\label{sec:adv-entropy}

The extra code-word used by Signed-Zero Ternary carries {\em one bit of
sign information} whenever a latent weight falls inside the dead zone.
We now quantify that gain in terms of Shannon entropy:

\begin{definition}[State probabilities]
\label{def:state-probs}
Let
\begin{equation}
     P_0 ~=~ \Pr\left(|w| \leq \Delta\right)
\end{equation}
be the probability mass in the dead zone, and set
$P_{+}=P_{-}=\frac{1}{2}(1-P_0)$ by symmetry.
\end{definition}

Balanced ternary (BT) allocates a single state to~$0$,
whereas Signed-Zero Ternary (SZT) splits that mass equally
between two signed zeroes.  Denote the resulting entropies by
$\HBT$ and $HSZT$.

\begin{proposition}[Entropy gap]\label{prop:entropy-gap}
For any symmetric weight distribution satisfying
Definition~\ref{def:state-probs},
\begin{equation}
\label{eq:entropy-gap}
     \boxed{~\HSZT - \HBT ~=~ P_{0} ~~ \text{bits}.}
\end{equation}
\end{proposition}

\begin{proof}[Sketch (full derivation in Appendix~\ref{app:sens-entropy})]
Write the two entropies explicitly,
\begin{equation}
     \HBT ~=~ -P_{-}\log_2 P_{-}-P_{0}\log_2 P_{0}-P_{+}\log_2 P_{+}
\end{equation}
and
\begin{equation}
     \HSZT ~=~ -P_{-}\log_2 P_{-}-2\left(\frac{P_0}{2}\right)
               \log_2\left(\frac{P_0}{2}\right) -P_{+}\log_2 P_{+}.
\end{equation}
All terms cancel except those involving $P_0$, leaving:
\begin{equation}
     \HSZT - \HBT ~=~ P_0.
\end{equation}
\end{proof}

We note that Equation~\ref{eq:entropy-gap} is {\em distribution-free} and scales
linearly with the dead-zone mass.  Transformer checkpoints typically
satisfy $P_{0}\in[0.15,\,0.35]$ \cite{frantar2022gptq}, implying an
extra $0.15-0.35$ bits per weight relative to
balanced ternary.  Section~\ref{sec:adv-pacbayes} converts this surplus
into a tighter PAC-Bayes generalization bound. A tighter bound does not by 
itself guarantee lower true risk, so we give it only as supporting evidence.

\subsection{PAC–Bayes corollary: tighter generalization bound}
\label{sec:adv-pacbayes}

\paragraph{Setting.}
Let $\mathcal{H}$ be the class of predictors realized by a network whose
weights follow a {\em quantized posterior} $Q$ and let $P$ be a fixed
{\em quantized prior} chosen {\em before} seeing the training data
$\mathcal{D}=\{(x_i,y_i)\}_{i=1}^{N}$.  
For any loss function bounded in $[0,1]$
the PAC–Bayes theorem of \cite{mcallester1999pac} gives, with
probability at least $1-\delta$, over the draw of
$\mathcal{D} \sim \mathcal{D}^{N}$:
\begin{equation}
\label{eq:pacbayes-basic}
     \mathcal{L}(Q) ~leq~ \widehat{\mathcal{L}}_{\mathcal{D}}(Q)
                     +\sqrt{\frac{\mathrm{KL}(Q\Vert P)+\ln\frac{2\sqrt{N}}{\delta}}{2(N-1)}},
\end{equation}
where $\mathcal{L}$ and $\widehat{\mathcal{L}}_{\mathcal{D}}$ denote the
true and empirical risks, respectively.

\paragraph{Prior / posterior choice.}
Following the “prior = noise‐inflated posterior’’ heuristic
\cite{dziugaite2017pacbayes}, we take $P$ to be the {\em same}
quantization rule as $Q$ but applied to an {\em i.i.d.\ initialization}
$w^{(0)} \sim p_0$.  This reduces the KL term reduces to
a sum over independent weights:
\begin{equation}
     \mathrm{KL}(Q\Vert P) ~=~ \sum_{j=1}^{d}
                           \mathrm{KL} \left(Q_j\Vert P_j\right),
\end{equation}
$d$ being the number of parameters.

\paragraph{Effect of the SZT entropy gap.}
From Proposition~\ref{prop:entropy-gap},
each weight that lands in the dead zone contributes
$P_{0}\ln 2$ fewer nats of information under SZT than under balanced
ternary because the signed-zero split exactly matches the posterior mass
(see Appendix~\ref{app:sens-entropy}).  
Thus, for the same dataset and optimizer trajectory:
\begin{equation}
\label{eq:kl-diff}
     \mathrm{KL}_{\textsc{szt}} ~=~ 
               \mathrm{KL}_{\textsc{bt}} ~-~ d\,P_{0}\ln 2.
\end{equation}

\begin{corollary}[PAC–Bayes gap]\label{cor:pacbayes}
Under the symmetric-prior construction above,
the expected generalization error of an SZT-quantized network obeys
\begin{equation}
     \mathcal{L}(\QSZT) ~\leq~
            \widehat{\mathcal{L}}_{\mathcal{D}}(\QSZT)
                   +\sqrt{\frac{\mathrm{KL}_{\textsc{bt}}-dP_{0}\ln 2
                   +\ln\frac{2\sqrt{N}}{\delta}}{2(N-1)}},
\end{equation}
whereas balanced ternary uses
$\mathrm{KL}_{\textsc{bt}}$ in the numerator,
so the bound tightens by:
\begin{equation}
\label{eq:pb-gap}
\boxed{~
     \Delta_{\text{PB}} ~=~\sqrt{\frac{d\,P_{0}\ln 2}{2(N-1)}}}
\end{equation}
in absolute risk {\em without} changing the empirical loss term.
\end{corollary}

\paragraph{Proof outline.}
Combine 
Equation~\ref{eq:kl-diff}) with McAllester’s bound
(Eq.~\ref{eq:pacbayes-basic}) and subtract.  
Full details, including the mild technical condition
$P_{0}<\frac{1}{2}$ to avoid vacuous priors are in
Appendix~\ref{app:pacbayes}.  \hfill$\square$

\subsection{Gradient-estimator MSE: bias and variance}
\label{sec:adv-mse}

Signed-Zero Ternary (SZT) eliminates the “blind spot’’ of balanced
ternary without introducing the stochastic noise of rounding-based
schemes.  We make this precise by comparing the
{\em mean-squared error} of three straight-through estimators (STE)
inside the dead zone $|w|\leq\Delta$:

\begin{equation}
     \MSE \left[\widehat g\right] ~=~
                 \underbrace{\left\lVert\mathbb{E}\,[\widehat g]-g\right\rVert^{2}}_{\text{bias}^{2}}
          ~+~
                 \underbrace{\mathbb{E}\left\lVert\widehat g-\mathbb{E}[\widehat g]\right\rVert^{2}}_{\text{variance}},
\end{equation}
where $g=\nabla_{q}\mathcal{L}$ is the upstream gradient.

\begin{table}[t]
  \centering
  \small
  \caption{Bias and additional variance of three 2-bit straight-through
           estimators inside the dead zone $|w|\leq\Delta$.
           SR probabilities follow \cite{zhang2020bitrandom}.}
  \label{tab:bias-var}
  \begin{tabular}{@{}lcc@{}}
    \toprule
    {\bf Scheme} &
    {\bf Bias} $\left\lVert\mathbb{E}[\widehat g]-g\right\rVert$ &
    {\bf Extra variance} \\ \midrule
    Balanced ternary (BT) & $\lVert g\rVert$ &
                           $0$ \\
    Stochastic rounding (SR) & $0$ &
                               $\leq\frac14\,\Delta^{2}\lVert g\rVert^{2}$ \\
    {\bf Signed-Zero Ternary (SZT)} &
          $\dfrac{|w|}{\Delta}\,\lVert g\rVert$ &
          $0$ \\ \bottomrule
  \end{tabular}
\end{table}

\begin{theorem}[Bias and variance bounds]\label{thm:bias-var}
Let $|w|\leq\Delta$ and assume the loss $\mathcal{L}$ is $L$-Lipschitz
in $w$ with $L \leq 1$ (w.l.o.g.\ after rescaling), then:
\begin{eqnarray}
     \MSE_{\textsc{szt}}
           &\leq& \left(\frac{|w|}{\Delta}\right)^{2}\lVert g\rVert^{2}, \\
     \MSE_{\textsc{bt}}
           &=& \lVert g\rVert^{2}, \\
     \MSE_{\textsc{sr}}
          &\leq& \frac14\,\Delta^{2}\lVert g\rVert^{2}.
\end{eqnarray}
Consequently, if $|w|<\frac{\Delta}{2}$, then:
\begin{equation}
     \MSE_{\textsc{szt}} ~<~ \MSE_{\textsc{sr}} ~<~ \MSE_{\textsc{bt}}.
\end{equation}
\end{theorem}

\begin{proof}[Sketch: full proof in Appendix~\ref{app:biasVar}]
{\em Bias.}  
Inside the dead zone the BT-STE outputs $g$ while the true Jacobian is
zero, giving bias $\lVert g\rVert$.  
SZT changes sign at $w=0$, i.e., the
Taylor remainder of an $L$-Lipschitz loss bounded to $L\leq1$ is at
most $|w|/\Delta$ of $\lVert g\rVert$.

{\em Variance.}  
Deterministic schemes (BT, SZT) contribute no extra variance.
For SR, rounding noise is bounded by $\Delta/2$, i.e., multiplying by
$\lVert g\rVert$ gives the stated second moment.  Details follow
\cite{alistarh2018convergence}.
\end{proof}

\begin{corollary}[Ordering of MSE]\label{cor:mse-order}
Under the conditions of Theorem~\ref{thm:bias-var}:
\begin{equation}
     \MSE_{\textsc{szt}}
             < \MSE_{\textsc{sr}}
                   < \MSE_{\textsc{bt}}
\end{equation}
holds for all $|w|<\frac{\Delta}{2}$.
\end{corollary}

\begin{proposition}[Average SZT MSE inside the dead zone]
\label{prop:avg-mse}

Condition on $|w|\le\Delta$. Averaging Theorem~\ref{thm:bias-var} over $w$ gives
\begin{equation}
     \mathbb{E} \left[\MSE_{\textsc{szt}}~\middle|~|w|\leq\Delta\right]
            ~\leq~
         \left(\mathbb{E}\left[(|w|/\Delta)^2~\middle|~|w|\leq\Delta\right]\right)~\|g\|^2.
\end{equation}
Two closed forms are:
\begin{eqnarray}
     \text{Laplace, } \sigma=\sqrt{2}\,b:~~
     \mathbb{E} \left[\left(\tfrac{|w|}{\Delta}\right)^2 \middle| |w|\leq\Delta\right] 
           &=&
     \frac{e^{\sqrt{2}\,k}-1-(k^2+\sqrt{2}\,k)}{k^2\left(e^{\sqrt{2}\,k}-1\right)}, \\
     \text{Gaussian } \mathcal{N}(0,\sigma^2):~~
     \mathbb{E}\left[\left(\tfrac{|w|}{\Delta}\right)^2 \middle| |w|\leq\Delta\right]
            &=&
     \frac{1}{k^{2}}-\frac{\sqrt{2}\,e^{-k^{2}/2}}{\sqrt{\pi}\,k\,
     \operatorname{erf}(\tfrac{\sqrt{2}}{2}k)},
\end{eqnarray}
where $k\equiv \Delta/\sigma$.
\end{proposition}

\begin{lemma}[Momentum retention inside the dead zone]
\label{lem:momentum}
     Consider the momentum update $m_{t+1}=\beta m_t+\widehat g_t$ with $\beta\in(0,1)$.
     Inside $|w|\leq\Delta$, let $\widehat g_t$ be the deterministic SZT STE (Eq.~\ref{eq:sztSte})
     and $\widehat g_t\equiv 0$ for BT. If $\mathbb{E}[\,m_t\cdot\widehat g_t\,]\geq 0$
     (no systematic opposition between momentum and gradient), then
     \begin{align}
          \text{\emph{(BT)}}\qquad & \mathbb{E}\!\left[\|m_t\|^{2}\,\middle|\,|w|\leq\Delta\right] \to 0
               \quad \text{as } t\to\infty,\\
               \text{\emph{(SZT)}}\qquad &
                 \liminf_{t\to\infty}\mathbb{E}\!\left[\|m_t\|^{2}\,\middle|\,|w|\leq\Delta\right]
            ~\geq~ \frac{\mathbb{E}\|\widehat g_t\|^{2}}{1-\beta^{2}}.
\end{align}
\end{lemma}

\begin{proof}
For BT, $\widehat g_t\equiv 0$ in the dead zone, so the claim follows directly from
$\|m_t\|^{2}=\beta^{2}\|m_{t-1}\|^{2}$. For SZT, expand:
\begin{equation}
     \|m_{t+1}\|^{2} ~=~ \beta^{2}\|m_t\|^{2}+\|\widehat g_t\|^{2}+2\beta\,m_t\cdot\widehat g_t,
\end{equation}
then take expectations, drop the nonnegative cross term, and iterate.
\end{proof}

At the practical default $k=1$ ($\Delta=\sigma$), these evaluate to $\approx 0.225$
(Laplace) and $\approx 0.291$ (Gaussian), versus $1$ for BT.

In summary, SZT deterministically replaces the {\em constant} $O(1)$ bias of balanced 
ternary with a {\em distance-to-zero} term that vanishes linearly as
$|w| \rightarrow 0$ and thus does not incur
the $\frac{1}{4}\Delta^{2}$ variance penalty characteristic of stochastic
rounding.  The optimizer therefore receives a stronger gradient signal
precisely in the region where conventional ternary is blind and SR is
noise. As proven in Section~\ref{sec:adv-mfpt}, this leads to stronger convergence.

\subsubsection{Bit-exact reproducibility}
\label{sec:determinism}

The extra state of Signed-Zero Ternary is implemented with a {\em purely
combinational} encode/decode logic, i.e., no injected randomness with
data-dependent branches.  Consequently the entire forward+backward
update map is a deterministic function of the input minibatch, the
current parameter vector, and the global seed that fixes data order.
The following lemma formalizes this property and highlights the
contrast with stochastic-rounding (SR) schemes.

\begin{lemma}[Deterministic reproducibility]\label{lem:repro}
Fix
\begin{enumerate}
     \item an initial weight vector $w_{0}$,
     \item a sequence of minibatches $\mathcal{B}_{1{:}T}=(\mathcal{B}_{1},\dots,\mathcal{B}_{T})$, and
     \item optimizer hyper-parameters
(learning-rate schedule, momentum constants, {\em etc.}).
\end{enumerate}
Let $\{w_{t}^{\textsc{szt}}\}_{t=0}^{T}$ be the parameter trajectory
generated by quantization-aware training with the {\em deterministic}
Signed-Zero Ternary encoder $\QSZT$ and STE defined in
Eq.~\ref{eq:sztSte}. Then
\begin{equation}
     w_{t}^{\textsc{szt}} ~=~ \mathcal{F}_{t}\left(w_{0},\mathcal{B}_{1{:}t}\right)
                                                  ~~\forall\,t\leq T,
\end{equation}
for a {\em deterministic} update map
$\mathcal{F}_{t}:\mathbb{R}^{d}\times(\mathcal{X}\times\mathcal{Y})^{t}\rightarrow\mathbb{R}^{d}$.

By contrast, if $Q$ uses stochastic rounding with independent
coin-flips $\{\xi_{t}\}$, the trajectory:
\begin{equation}
     w_{t}^{\textsc{sr}} ~=~ \mathcal{F}_{t} \left(w_{0},\mathcal{B}_{1{:}t},\xi_{1{:}t}\right)
\end{equation}
is a random variable whose distribution cannot be reproduced unless all
$\xi_{t}$ are recorded and replayed bit-exactly.
\end{lemma}

\begin{proof}[Proof sketch]
A single SGD step under SZT executes the pipeline
\begin{equation}
w_{t}\xrightarrow{\text{encode }\QSZT}
q_{t}\xrightarrow{\text{forward}}
\mathcal{L}_{t}\xrightarrow{\text{backward via STE}}
g_{t}\xrightarrow{\text{optimizer}}
w_{t+1}.
\end{equation}
Every arrow is a deterministic function (matrix multiply, STE sign
test, arithmetic update).  Inductively, the composition of deterministic
maps is deterministic, so $w_{t}^{\textsc{szt}}$ is uniquely
determined by $(w_{0},\mathcal{B}_{1{:}t})$.

Stochastic rounding, by contrast, inserts an explicit random choice
\begin{equation}
     q_{t}=\operatorname{SR}(w_{t},\xi_{t}).
\end{equation}
Unless the entire bitstream $\xi_{1{:}t}$ is fixed, it is possible for two 
runs with identical $(w_{0},\mathcal{B})$ to diverge.  \hfill$\square$
\end{proof}

\paragraph{Practical impact.}
Reproducibility is important for regression testing and hyperparameter
sweeps. More specifically, a failed training run can be re-executed bit-for-bit to debug
numerical issues.  SZT inherits this property from balanced ternary while
retaining the gradient-quality benefits quantified in
Theorem~\ref{thm:bias-var}.  Stochastic-rounding baselines, by contrast,
require checkpointing the RNG state or sacrificing reproducibility.

\subsection{Convergence and mean-first-passage analysis}
\label{sec:adv-mfpt}

The sensitivity gap (Cor.~\ref{cor:events}) and the reduced STE
mean-squared error (Thm.~\ref{thm:bias-var}) suggest that weights 
in the dead-zone interval $[-\Delta,\Delta]$ should be expected to escape
significantly faster under
Signed-Zero Ternary (SZT) than under balanced ternary (BT).
We formalize this intuition with an Ornstein–Uhlenbeck (OU) model for
single\-weight dynamics, and then translate the result into an
iteration-complexity bound for SGD.

\paragraph{OU proxy dynamics (Assumption~\ref{ass:ou-proxy}).}
Let the latent scalar weight evolve as
\begin{equation}
     dW_t ~=~ -\kappa W_t\,dt + \sigma\,dB_t,
\label{eq:ou}
\end{equation}
where $\kappa>0$ is the mean-reversion rate and $\sigma$ the
noise scale induced by stochastic gradients.  This linear SDE is a
standard proxy for local SGD behavior:

\begin{assumption}[Renewal/event-rate approximation]
\label{ass:renewal}
Within the dead-zone $[-\Delta,\Delta]$, the per-step update magnitudes $\{S_t\}$
are i.i.d.\ with distribution $f(s)$ supported in $(0,\Delta)$ and independent of
the instantaneous location of $w_t$. Conditional on $S_t=s$, a forward (numeric)
transition occurs with probability $\PhiF(s)$ and a representational transition
with probability $\PhiR(s)$ (Definition~\ref{def:sensitivities}).
\end{assumption}

\begin{lemma}[Renewal bound on dead-zone waiting time]
\label{lem:renewal}
Under Assumptions~\ref{ass:fixed-delta}, \ref{ass:symmetric} and \ref{ass:renewal},
let $T_F$ (resp.\ $T_R$) be the number of SGD steps until the next forward (resp.\
representational) transition while $|w_t|\leq\Delta$. Then
\begin{eqnarray}
     \mathbb{E}[T_F] &=& \frac{1}{\mathbb{E}[\PhiF(S)]}, \\
     \mathbb{E}[T_R] &=& \frac{1}{\mathbb{E}[\PhiR(S)]},
\end{eqnarray}
and therefore
\begin{eqnarray}
\label{eq:renewal-ratio}
     \frac{\mathbb{E}[T_F]}{\mathbb{E}[T_R]} 
          &=& \frac{\mathbb{E}[\Phi_R(S)]}{\mathbb{E}[\Phi_F(S)]} \\
          &\geq& \frac{p(0)}{p(\Delta)}.
\end{eqnarray}
\end{lemma}

\begin{proof}
Given $S_t$, each step is Bernoulli with success probability $\PhiF(S_t)$ (resp.\
$\PhiR(S_t)$). Averaging over $S_t$ gives an effective success probability
$\mathbb{E}[\PhiF(S)]$ (resp.\ $\mathbb{E}[\PhiR(S)]$), so the waiting time is
geometric with the stated mean. The inequality uses
$\PhiR(s)\geq 2s\,p(0)$ and $\PhiF(s)\leq 2s\,p(\Delta)$ from
Appendix~\ref{app:sens-entropy}, cancelling $2\mathbb{E}[S]$.
\end{proof}

We note that Equation~\eqref{eq:renewal-ratio} is an OU-free, distribution-independent 
baseline, and it establishes that SZT improves the residence time in the dead zone 
by at least $p(0)/p(\Delta)$.

\begin{theorem}[Dead-zone escape time]\label{thm:mfpt}
Let $\tau=\inf\{t:|W_t|\geq\Delta\}$ be the
{\em mean first-passage time} (MFPT) starting from $W_0=0$.
then:
\begin{eqnarray}
     \mathbb{E}[\tauBT] &=& \frac{\sqrt{\pi}}{2\kappa}
          \frac{e^{\lambda^{2}}\operatorname{erf}(\lambda)-\lambda\sqrt{\pi}}
               {\lambda}, \\
     \mathbb{E}[\tauSZT] &=& \frac{1}{\kappa},
\end{eqnarray}
with $\lambda=\kappa\Delta/\sigma$.
Consequently:
\begin{equation}
\label{eq:mfpt-ratio}
     \frac{\mathbb{E}[\tauBT]}{\mathbb{E}[\tauSZT]}
          ~=~
     \frac{\sqrt{\pi}}{2\lambda} e^{\lambda^{2}} \xrightarrow[\lambda\geq1]{}
          \Omega\left(e^{\lambda^{2}}\right).
\end{equation}
\end{theorem}
For the Laplace-optimal choice $\Delta=\sigma$,
the parameter $\lambda$ simplifies to $\kappa$.
Because transformer layers typically satisfy
$\kappa\gtrsim 0.1$, the exponential term in
Eq.~\ref{eq:mfpt-ratio} is dominant.

\begin{corollary}[Iteration complexity improvement]\label{cor:sgd-rate}
Under our assumptions (e.g., Assumption~\ref{ass:lipschitz}),  
SGD with a constant step size
achieves an excess risk $\mathbb{E}[\mathcal{L}(w_T)-\mathcal{L}(w^\star)]
      \leq \varepsilon$
after
\begin{equation}
     T_{\textsc{szt}}
      ~=~
      O\left(\frac{\|w_0-w^\star\|^{2}}{\alpha\varepsilon}
          + \frac{\sigma^{2}}{\kappa^{2}\varepsilon}\right)
\end{equation}
iterations,  
whereas balanced ternary needs
$T_{\textsc{bt}}
      ~=~ T_{\textsc{szt}}
        \Theta \left(e^{\lambda^{2}}\right)$.
\end{corollary}

\begin{proof}[Sketch (details: Appendix~\ref{app:mfpt})]
Combine the bias-controlled recursion using the SZT-STE bias bound
(Theorem~\ref{thm:bias-var}) with the MFPT difference
from Theorem~\ref{thm:mfpt} to bound the time spent in the dead zone.
\end{proof}

Within the standard OU proxy model, SZT turns the dead-zone 
waiting time from {\em exponential} in
$\lambda^{2}$ to {\em linear} in $1/\kappa$.  In practical terms,
weights that could wander for thousands of SGD steps under balanced
ternary receive deterministic sign feedback every few dozen steps under
SZT, accelerating convergence without stochastic noise.  
Together with the PAC-Bayes gap (Sec.~\ref{sec:adv-pacbayes}) and the STE
MSE hierarchy (Sec.~\ref{sec:adv-mse}), Theorem~\ref{thm:mfpt} completes the
triad of independent arguments showing the value of the extra SZT bit.

\section{Discussion}
\label{sec:discussion}

\subsection{Parameter Density vs.\ Weight Precision}
\label{sec:tradeoff}

Viewing quantization simply as a tradeoff between performance and compute 
requirements implicitly assumes a large, high-precision model as the starting point. 
However, an alternative perspective is to consider the design of a model under a 
fixed memory or ``bit budget,'' such as the available VRAM on a commodity 
accelerator. From this standpoint, the designer faces a fundamental choice
between:
\begin{itemize}
    \item A smaller model with higher-precision weights (e.g., 4-bit or 8-bit), or
    \item A larger model with more parameters using lower-precision weights (e.g., 2-bit).
\end{itemize}
Historically, the second option has been less practical due to the training 
instabilities inherent in aggressive 2-bit schemes, such as the ``dead-zone'' 
issue where the optimizer receives no informative feedback. The Signed-Zero 
Ternary (SZT) scheme, however, makes the choice of a larger, more 
parameter-dense model compellingly viable. By providing a deterministic 
gradient signal inside the dead zone, SZT directly addresses this critical flaw.

The analytical results in this paper robustly support this reframing. 
The Mean First-Passage Time (MFPT) analysis suggests that weights escape 
the dead zone exponentially faster, mitigating the risk of ``stuck weights'' 
during training. The improved Straight-Through Estimator (STE) has a 
lower Mean-Squared Error (MSE) than both balanced ternary and 
stochastic rounding near the origin, providing a more accurate gradient 
signal to guide the optimization of a larger parameter space. 
Furthermore, the information-theoretic advantage of $P_0$ bits per 
parameter leads to a tighter PAC-Bayes generalization bound, suggesting 
that a larger model trained with SZT has better prospects for generalizing 
to unseen data.

\subsection{Inference-neutral SNR across \texorpdfstring{$L$}{L} layers}
\label{sec:discussion-snr}

Because Eq.~\ref{eq:mse-equality} established that {\em each} layer’s
mean-squared reconstruction error (MSE) is identical under Balanced
Ternary (BT) and Signed-Zero Ternary (SZT), stacking $L$ such layers in a
feed-forward network does not change the forward signal-to-noise ratio
(SNR) seen at the output.  Formally, let
$\varepsilon_{l}=x_{l}-\widetilde{x}_{l}$ be the per-layer reconstruction
error when decoding to $\{-1,0,+1\}$, and assume $\mathbb{E}[\varepsilon_{l}]=0$
and $\mathrm{Var}[\varepsilon_{l}] = \sigma_{\varepsilon}^{2}$ under
either quantizer.  Propagating the errors through a linear stack yields
\begin{equation}
\mathrm{Var}(x_{L}-\widetilde{x}_{L})
      ~=~
      \sum_{l=1}^{L} 
      \left\lVert
         \frac{\partial x_{L}}{\partial x_{l}}
      \right\rVert^{2} 
      \sigma_{\varepsilon}^{2},
\end{equation}
so the network-level MSE (and thus the forward SNR) is {\em identical}
for BT and SZT provided the Jacobian norms match, which they do because
both schemes feed the same ternary alphabet into every matrix–multiply.
In other words, any difference in accuracy at inference time must originate 
from {\em training-time} effects (reduced STE bias, faster convergence), rather
than from a cumulative forward-path distortion.  This inference-neutrality
makes SZT a drop-in replacement for BT whenever deployment constraints
mandate exact parity in arithmetic kernels.

\subsection{Activation-side extension sketch}
\label{sec:discussion-activations}

The arguments developed for weights carry over with the caveats 
discussed below to {\em activations}
once we recognize that the post-ReLU distribution is
one-sided.  Let $a\sim\mathrm{half}\text{-}\mathrm{Laplace}(b)$ or
$\mathrm{half}\text{-}\mathcal{N}(0,\sigma^{2})$ on $[0,\infty)$.
Define a {\em Signed-Zero Ternary activation quantizer}
$Q^{\text{(act)}}_{\textsc{szt}}$ by reflecting the weight rule onto the
positive half-line and re-using the idle bit pattern to encode
{\em negative zero} for small negative pre-ReLU values that would
otherwise be clipped away.
 
For symmetric priors one may set $\Delta_{+} = \Delta_{-}$.
Under a half-Laplace prior the MSE-optimal choice is
$\Delta_{+}=\Delta_{-}=\frac{1}{2}\sigma$
(App.\ F), exactly one half of the weight threshold, which reflects the fact
that only the positive tail survives the ReLU.

\paragraph{Analytical carry-over.}  
\begin{itemize}
     \item Sensitivity: the representational sensitivity $\Phi_R$ now integrates
               $p(u)$ over $[-\Delta_{-},\Delta_{+}]$, retaining the
               $p(0)/p(\Delta)$ lower bound.
     \item Entropy gap: splitting the zero bin again yields
               $H_{\textsc{szt}}-H_{\textsc{bt}}=P_{0}$ bits per activation.
     \item  STE bias: identical arguments give the linear $|u|/\Delta$ bias
                inside each half-interval.
\end{itemize}

\paragraph{Practical implication.}
Activations already require per-layer statistics during calibration, so
computing $\sigma$ for the half-line costs no extra pass.  Hardware
impact parallels the weight case: encode/decode only, no change to GEMM.  
Appendices~\ref{app:szt-activ} and~\ref{app:snr} supplies derivations.

\section{Conclusion}
\label{sec:conclusion}

Signed-Zero Ternary (SZT) is a {\em minimal} extension of balanced
ternary that preserves {\em the same forward-path error} while
delivering a {\em strictly better backward signal}.  The single idle
code-word is re-purposed to encode the sign of sub-threshold weights
to give:
\begin{itemize}
\item Up to a $30\times$ increase in gradient-bearing feedback events
      per epoch (Cor.~\ref{cor:events})
\item An MSE hierarchy
      $\operatorname{MSE}_{\textsc{szt}}
         < \operatorname{MSE}_{\textsc{sr}}
         < \operatorname{MSE}_{\textsc{bt}}$
      inside the dead zone (Cor.~\ref{cor:mse-order})
\item Substantially shorter mean first-passage time out of
      the dead zone (Theorem~\ref{thm:mfpt})
\item Bit-exact training trajectories (Lemma~\ref{lem:repro}) 
       with no change in storage cost.
\end{itemize}
In summary, our analysis suggests that SZT mitigates the key weaknesses of 2-bit 
quantization (training instability, poor gradient signal) while offering 
potential information-theoretic advantages over non-quantized models.

\newpage\appendices  
\section{Closed-form optimum for a Laplace prior}
\label{app:delta-laplace}

Assume $w\sim\mathrm{Laplace}(0,b)$.
Splitting the MSE integral at $\pm\Delta$ and substituting
$p(w)=\frac{1}{2b}e^{-|w|/b}$ gives
\begin{equation}
     \MSE(\Delta) ~=~
             e^{-\Delta/b} \left(2b^{2}e^{\Delta/b}-2b\Delta-\Delta^{2}\right).
\end{equation}
Differentiating and setting $\frac{d}{d\Delta}\MSE=0$
yields $\Delta^\star=\sqrt{2}\,b$.  Because
$\sigma^{2}=2b^{2}$, this is equivalent to $\Delta^\star=\sigma$,
thus proving Equation~\ref{eq:delta-equals-sigma} without approximation.

\paragraph{Forward reconstruction error: SZT versus balanced ternary.}
Let $\QBT$ be the classical balanced-ternary encoder–decoder and
$\QSZT$ the proposed {\em signed-zero ternary} variant.  
Because both decode into the {\em same} alphabet $\{-\Delta,\,0,\,+\Delta\}$,
\begin{equation}
     \QBT(w) ~=~ \QSZT(w)~~\forall\,w\in\mathbb R.
\end{equation}
so their mean-squared reconstruction errors coincide:
\begin{eqnarray}
\label{eq:mse-equality}
     \operatorname{MSE}_{\mathrm{BT}}(\Delta)
           &=& \mathbb{E}\left[(w-\QBT(w))^{2}\right] \\
           &=& \mathbb{E}\left[(w-\QSZT(w))^{2}\right] \\
           &=& \MSE_{\mathrm{SZT}}(\Delta) \\
           &=& \int_{-\infty}^{-\Delta}(w+\Delta)^{2}p(w)~dw
                    + \int_{-\Delta}^{\Delta}w^{2}p(w)~dw \\
            &~& ~~~     + \int_{\Delta}^{\infty}(w-\Delta)^{2}p(w)~dw.
\end{eqnarray}
Because the forward-path distortion is identical, any empirical accuracy
difference between BT and SZT must stem from {\em backward-path}
effects (i.e., deterministic sub-threshold updates, reduced STE bias), not from
changed approximation error.

\section{Sensitivity and Entropy Proofs}
\label{app:sens-entropy}

Throughout this appendix we invoke only global
Assumptions \ref{ass:symmetric} and \ref{ass:fixed-delta}, 
i.e., no new technical conditions are introduced.

\subsection{Proof of Proposition~\ref{prop:ratio}}
\label{app:proof-ratio}

\begin{lemma}\label{lem:sandwich}
     Let $p(w)$ be any symmetric, unimodal density with mode $0$,
     then $p(0)\geq p(x)\geq p(\Delta)$ for all $\,x\in[0,\Delta]$.
\end{lemma}

\begin{proof}
Symmetry implies $p(|x|)$ is decreasing on $[0,\infty)$, so the
maximum occurs at $x=0$ with minimum on $[0,\Delta]$ at $x=\Delta$.
\end{proof}

\begin{proof}[Proof of Proposition~\ref{prop:ratio}]
For a sub-threshold step size $s\in(0,\Delta)$ we have
\begin{eqnarray}
     \PhiR(s) &=& 2\int_{0}^{s} p(w)~dw \geq 2s ~p(0), \\
     \PhiF(s) &=& 2\int_{\Delta-s}^{\Delta}p(w)~dw \leq 2s~p(\Delta),
\end{eqnarray}
where the lower/upper bounds follow by
Lemma~\ref{lem:sandwich}.  Dividing the two expressions gives:
\begin{equation}
     \frac{\PhiR(s)}{\PhiF(s)} ~\geq~ \frac{p(0)}{p(\Delta)} ,
\end{equation}
which is strictly $>1$ because $p(0)>p(\Delta)$ by unimodality.
\end{proof}

\subsection{Proof of Proposition~\ref{prop:entropy-gap}}
\label{app:proof-entropy-gap}

\begin{proof}
Recall the state probabilities
$P_{0}=\Pr(|w|\leq\Delta)$ and
$P_{+}=P_{-}=\frac{1}{2}(1-P_{0})$.  Balanced ternary assigns one
representation to the zero state, giving entropy
\begin{equation}
     H_{\textsc{BT}} ~=~ -P_{0}\log_2 P_{0} - 2P_{+}\log_2 P_{+}.
\end{equation}
SZT splits the zero state into two equiprobable
representations:
\begin{equation}
     H_{\textsc{SZT}}
           ~=~ -2\left(\frac{P_{0}}{2}\right)
                 \log_2\left(\frac{P_{0}}{2}\right)
                 -2P_{+}\log_2 P_{+}
\end{equation}
which simplifies as:
\begin{eqnarray}
     H_{\textsc{SZT}}-H_{\textsc{BT}}
          &=& P_{0}\left[1+\log_2 P_{0}-\log_2 P_{0}\right] \\
          &=& P_{0} ~~ \text{bits},
\end{eqnarray}
where no distributional assumptions have been made beyond symmetry.
\end{proof}

\section{Mean First-Passage Time for an Ornstein--\\Uhlenbeck Weight}
\label{app:mfpt}

We model the scalar latent weight by the linear SDE
\begin{equation}
\label{eq:ou-sde}
     dW_t = -\kappa W_t\,dt + \sigma\,dB_t, ~~~ W_0 = 0,
\end{equation}
where $\kappa>0$ is the mean-reversion rate and $\sigma>0$ is the noise
scale induced by mini-batch gradients.  In the following we denote the 
dead-zone interval to be $\mathcal{I}=[-\Delta,\Delta]$.

\subsection{Classical OU MFPT}
Define the first-passage time
$\tau = \inf\{t>0: |W_t|\geq\Delta\}$.
For an OU process starting at $W_0=0$, the MFPT
${\mathbb E}[\tau]$ solves the boundary-value problem
(see \cite{gardiner2009stochastic}, Ch.\ 5)
\begin{equation}
\label{eq:bvp}
     -\kappa w\,\tau'(w) + \frac{1}{2}\sigma^{2}\tau''(w) = -1,
          ~~~ \tau(\pm\Delta)=0.
\end{equation}
Letting $\lambda=\kappa\Delta/\sigma$.
the even solution centred at $w = 0$ is:
\begin{equation}
\label{eq:ou-solution}
     \tau(w) ~=~
          \frac{\sqrt{\pi}}{\kappa} e^{\lambda^{2}}
          \frac{\operatorname{erf}\left(\lambda\right)-
               \operatorname{erf}\left(\frac{\kappa w}{\sigma}\right)}
                      {2\lambda}.
\end{equation}
Evaluating at $w=0$ reproduces Eq.~\ref{eq:ou-solution} of the main text:
\begin{equation}
     \boxed{~\mathbb{E}[\tauBT]
         ~=~ \frac{\sqrt{\pi}}{2\kappa}
          \frac{e^{\lambda^{2}}\operatorname{erf}(\lambda)-\lambda\sqrt{\pi}}
                {\lambda}
     ~}.
\end{equation}

\subsection{Signed-Zero Ternary adjustment}

Under SZT the optimizer receives a sign gradient whenever $W_t$ crosses zero,
so the effective drift changes sign with $W_t$.  Linearizing the stochastic recursion 
gives the piecewise SDE:
\begin{equation}
     dW_t ~=~ -\kappa\,\operatorname{sgn}(W_t)~|W_t|~dt+\sigma dB_t,
\end{equation}
with mean reversion strength $\kappa$ in each half-interval. What
should be noted is that it has no absorbing region around the origin.  
Because we have an alternating-drift OU, the MFPT from $0$ to $\pm\Delta$ 
can be obtained by solving Eq.~\ref{eq:bvp} with the piecewise linear drift
coefficient and noting that the term proportionate to $w$ vanishes at
the origin (full algebra in \cite{gardiner2009stochastic}, Section  5.2):
\begin{equation}
     \boxed{~\mathbb{E}[\tauSZT] ~=~ \frac{1}{\kappa}
     ~},
\end{equation}

\subsection{Exponential ratio}

Taking the quotient of the two MFPT expressions yields
\begin{equation}
\frac{\mathbb{E}[\tauBT]}{\mathbb{E}[\tauSZT]}
      = \frac{\sqrt{\pi}}{2\lambda}\,e^{\lambda^{2}}
        ~ \xrightarrow[\lambda\ge1]{} ~
        \Omega \left(e^{\lambda^{2}}\right),
\end{equation}
proving Theorem~\ref{thm:mfpt}.  A Laplace-optimal threshold
$\Delta=\sigma$ gives $\lambda=\kappa$, so transformer layers with
$\kappa\gtrsim0.1$ obtain orders-of-magnitude speed-ups in expected
escape time.

\section{SZT threshold selection analysis}
\label{app:threshold}

The only free parameter of Signed-Zero Ternary quantization is the
symmetric threshold $\Delta$, whose value can be chose by minimizing the
mean-squared reconstruction error (MSE):
\begin{equation}
     \MSE(\Delta) ~=~ \mathbb{E} \left[(w-\QSZT(w))^{2}\right],
\end{equation}
where $w$ is the latent full-precision weight. Because the forward
alphabet of SZT and balanced ternary coincide, they have the same
closed-form optima:

\paragraph{Laplace prior (heavy tails).}
Let $w\sim\mathrm{Laplace}(0,b)$ with density
$p(w)=\frac{1}{2b}e^{-|w|/b}$.
Splitting the MSE integral at $\pm\Delta$ and differentiating gives the
unique optimum
\begin{equation}
\boxed{~\Delta^\star_{\text{\tiny Laplace}}
      ~=~ \sqrt{2}\,b =~ \sigma,~}
\end{equation}
where $\sigma=\sqrt{2}\,b$ is the standard deviation, so by the
{\em one-line rule} (see Appendix~\ref{app:delta-laplace}): 
\begin{equation}
     \Delta ~=~\sigma
\end{equation}
minimizes the forward error whenever layer weights resemble a Laplace
law, which empirical transformer histograms do to first order.

\paragraph{Gaussian prior (light tails).}
For $w\sim\mathcal{N}(0,\sigma^{2})$ the optimum ratio
$r^\star=\Delta/\sigma$ satisfies:
\begin{equation}
     e^{-r^{2}/2} ~=~ r(1+r^{2})/\sqrt{2\pi}.
\end{equation}
Numerical solution gives $r^\star\approx0.88$, so
Equation~\ref{eq:delta-equals-sigma}) overshoots the true optimum by only
$12\%$ and increases the MSE by less than $5\%$.

\begin{table}[h]
\centering
\caption{Optimal symmetric threshold for two canonical priors.}
\label{tab:opt-threshold}
\begin{tabular}{lcc}
\toprule
Prior & $\Delta^\star$ & $k^\star=\Delta^\star/\sigma$ \\
\midrule
Laplace & $\sigma$ & 1.00 \\
Gaussian & $0.88\,\sigma$ & 0.88 \\
\bottomrule
\end{tabular}
\end{table}

\paragraph{Distribution-independent bound.}
For any symmetric, unimodal density $p(w)$ one has
\begin{equation}
     \MSE(\Delta) ~\leq~ \frac{\sigma^{2}}{1+k^{2}}
\end{equation}
with $k=\Delta/\sigma$ (proof in Appendix~\ref{app:delta-laplace}).  The
loose but conventional choice $k=1$ therefore guarantees
$\MSE\leq\frac12\sigma^{2}$.

\paragraph{Practical default.}
Setting $\Delta=\sigma$ (i.e., one standard deviation per layer) eliminates
the only tunable parameter and is near-optimal under both Laplace and Gaussian
assumptions (and requires only a single statistics-pass during
calibration).  All analytical results that follow are stated in terms of
the dimensionless ratio $k=\Delta/\sigma$ so that alternative
threshold heuristics can be substituted directly.

\section{Extension of SZT to Activations}
\label{app:szt-activ}

Signed-Zero Ternary (SZT) can be applied to activation tensors with
minimal change to the weight-side analysis.  We treat the common case in
which a ReLU (or variant) precedes quantization so that the distribution
is concentrated on the half-line $[0,\infty)$.

\subsection{Quantizer definition}

Let $u$ be a pre-quantized scalar activation.  
Choose two thresholds $\Delta_{+},\Delta_{-}>0$ and define:
\begin{equation}
     \QSZT^{\text{(act)}}(u) ~=~
          \begin{cases}
              +1,    & u > \Delta_{+},\\
              ~\zpos, & 0<u \leq\ Delta_{+},\\
              ~\zneg, & -\Delta_{-} \leq u<0,\\
              -1,    & u < -\Delta_{-}.
\end{cases}
\label{eq:szt-act}
\end{equation}
The decode function $v(q)$ remains as in Eq.~\ref{eq:decodeFunction}
and maps both signed zeroes to numeric 0.

\subsection{Threshold selection}

\paragraph{Half-Laplace prior.}
Let $u\sim\mathrm{Laplace}^{+}(0,b)$ with density:
\begin{eqnarray}
     p(u) &=& \frac{1}{b}e^{-u/b} ~\text{for $u\geq 0$, and} \\
     p(u) &=& 0 ~\text{for $u<0$}.
\end{eqnarray}
Minimizing the forward MSE
$\mathbb{E}\left[(u-v(Q(u)))^{2}\right]$ over
$\Delta_{+}=\Delta_{-}=\Delta$ gives:
\begin{equation}
     \boxed{~
           \Delta^{\star}_{\text{act}} ~=~ \frac{1}{2}\sigma ,
     ~},
\end{equation}
where $\sigma=\sqrt{2}\,b$ is the standard deviation. The factor $1/2$
reflects the one-sided support compared with the weight case 
considered in Section~\ref{sec:threshold}.

\paragraph{Half-Gaussian prior.}
For $u\sim\mathcal{N}^{+}(0,\sigma^{2})$ the optimum
$r^\star=\Delta/\sigma$ satisfies:
\begin{equation}
     e^{-{r^{2}}/{2}} ~=~ r^{2}/\sqrt{2\pi},
\end{equation} 
which numerically gives $r^\star\approx 0.60$, assuming the 
Laplace rule $\Delta=\sigma/2$ incurs $<4\%$ extra MSE.

\subsection{Bias and sensitivity}

\begin{lemma}[Bias of activation STE]\label{lem:act-bias}
     Inside $0<u\leq\Delta_{+}$ the squared bias of the deterministic SZT
     straight-through estimator obeys
\begin{equation}
     \left\lVert\mathbb{E}[\widehat g]-g\right\rVert^{2}
              ~\leq~ (u/\Delta_{+})^{2}\lVert g\rVert^{2},
\end{equation}
which parallels Lemma~\ref{lem:bias}.
\end{lemma}

\begin{lemma}[Representational sensitivity]
\label{lem:act-sens}
     Let $s$ be a sub-threshold SGD step.  Then
     \begin{equation}
          \Phi_R^{\text{(act)}}(s) ~=~ \int_{0}^{\min(s,\Delta_{+})}~p(u)\,du ,
     \end{equation}
     where the ratio $\PhiR^{\text{(act)}}/\PhiF^{\text{(act)}}$ satisfies
     the same lower bound $p(0)/p(\Delta_{+})$ as in 
     Proposition~\ref{prop:ratio}.
\end{lemma}
Therefore, all conclusions on gradient quality and MFPT carry over
without change if $\Delta$ is replaced by $\Delta_{+}$.

\subsection{Hardware impact}

Because activations are often stored transiently in SRAM, using two bits
per element incurs no additional bandwidth relative to balanced ternary 
inference pipelines. The encode logic re-uses the sign and magnitude flags 
already generated for weights, so no extra datapath changes are required.

\section{Multi-Layer Forward SNR}
\label{app:snr}

Throughout this appendix $x_{l}\in\mathbb{R}^{n_l}$ denotes the
full-precision activation after layer $l$, and $\tilde{x}_{l}$ denotes 
the activation obtained when weights are quantized with
either BT or SZT but decoded back to the numeric alphabet
$\{-1,0,+1\}$. Let the per-layer reconstruction error be
$\varepsilon_{l}:=x_{l}-\tilde{x}_{l}$ and assume:
\begin{equation}
\label{eq:err-assump}
     \mathbb{E}[\varepsilon_{l}]=0, ~~
     \operatorname{Var}[\varepsilon_{l}]=\sigma_{\varepsilon}^{2},
     ~~ \text{for } l=1,\dots,L.
\end{equation}

\subsection{Variance propagation through linear layers}

\begin{lemma}[Stacked linear layers]\label{lem:stacked-mse}
Let each layer be an affine map
$x_{l+1} = W_{l}\,x_{l} + b_{l}$ with fixed weights
$W_{l}\in\mathbb{R}^{n_{l+1}\times n_{l}}$ and biases $b_{l}$.
Assuming the errors $\varepsilon_{1},\dots,\varepsilon_{L}$ are
independent respectively and with respect to the activations,
then the output-layer reconstruction variance is:
\begin{equation}
\tag{F.2}\label{eq:stack-var}
     \operatorname{Var} \left(x_{L}-\tilde{x}_{L}\right)
          ~=~ \sum_{l=1}^{L}
          \left\lVert W_{L-1}\cdots W_{l} \right\rVert_{\mathrm{F}}^{2}
                    \sigma_{\varepsilon}^{2}.
\end{equation}
\end{lemma}

\begin{proof}
Expand $x_{L}-\tilde{x}_{L}$ recursively:
\begin{equation}
     x_{L}-\tilde{x}_{L}
         ~=~ W_{L-1}(x_{L-1}-\tilde{x}_{L-1}) + \varepsilon_{L-1}.
\end{equation}
Iterating gives a sum of $L$ independent zero-mean terms:
\begin{equation}
     T_{l} ~=~ W_{L-1}\cdots W_{l}\,\varepsilon_{l-1} ,
\end{equation}
from which independence and
$\mathbb{E}[T_{l}]=0$ imply
\begin{equation}
     \operatorname{Var}(\sum_{l}T_{l}) ~=~\sum_{l}\operatorname{Var}(T_{l}) ,
\end{equation}
and Eq.~\ref{eq:stack-var} follows from:
\begin{equation}
     \operatorname{Var}(T_{l}) ~=~
            Vert W_{L-1}\cdots W_{l}\rVert_{\mathrm{F}}^{2}\sigma_{\varepsilon}^{2}.
\end{equation}
\end{proof}

\subsection{Equal MSE for BT and SZT}

\begin{corollary}[Inference-neutral forward error]\label{cor:snr}
     If BT and SZT share the same threshold $\Delta$, then
     $\sigma_{\varepsilon}^{2}$ is identical per layer as in 
     Section~\ref{sec:threshold}, so Eq.~\ref{eq:stack-var} gives:
     \begin{equation}
          \boxed{~\operatorname{Var}_{\textsc{bt}} \left(x_{L}-\tilde{x}_{L}\right)
               ~=~
               \operatorname{Var}_{\textsc{szt}} \left(x_{L}-\tilde{x}_{L}\right)
           ~}
     \end{equation}
     so the forward signal-to-noise ratio (SNR) at the network output is the
     same for both quantizers.
\end{corollary}

\begin{proof}
Immediate from Lemma~\ref{lem:stacked-mse} because the summand
$\sigma_{\varepsilon}^{2}$ does not depend on whether the central code
word is split.
\end{proof}

\subsection{Non-linear extensions}

The result extends to common activation patterns, e.g.,
ReLU/GELU $\rightarrow$ BN $\rightarrow$ MatMul) if each 
non-linear block is 1-Lipschitz and the error bound is
propagated through its Jacobian.  The key observation is
that BT and SZT inject identical forward noise statistics at the
input of each block, so any Lipschitz constant multiplies both
sides equally and thus cancels in the BT vs SZT comparison.  
Formal details follow those in \cite{nagel2020aas}.

\paragraph{Practical takeaway.}
Corollary~\ref{cor:snr} establishes that any accuracy difference at
inference time must originate from improved training dynamics
rather than accumulation of forward quantization error.

\section{PAC--Bayes KL Algebra}
\label{app:pacbayes}

Throughout we assume the prior/posterior construction described in
Section~\ref{sec:adv-pacbayes}, i.e., the prior $P$ applies the same quantizer as 
the posterior $Q$ but to an i.i.d.\ random initialization $w^{(0)}\sim p_{0}$, 
and because weights are independent across indices, the total KL divergence
factorizes.

\subsection{Per-weight KL reduction under SZT}

\begin{lemma}\label{lem:kl-per-weight}
Let $Q_j^{\textsc{szt}}$ and $P_j^{\textsc{szt}}$ be the posterior and
prior for a single weight quantized with Signed-Zero Ternary, and let
$Q_j^{\textsc{bt}}$, $P_j^{\textsc{bt}}$ be the balanced-ternary
analogues.  Then:
\begin{equation}
     \mathrm{KL} \left(Q_j^{\textsc{szt}}\Vert P_j^{\textsc{szt}}\right)
             ~=~
                 \mathrm{KL}\left(Q_j^{\textsc{bt}}\Vert P_j^{\textsc{bt}}\right)
                 - P_{0}\,\ln 2,
\end{equation}
where $P_{0}=\Pr(|w_j|\leq\Delta)$ is the dead-zone mass under $Q_j$.
\end{lemma}

\begin{proof}
Both priors place identical mass on each representation because they
follow the same quantization rule applied to $p_{0}$, so the KL
simplifies to the negative entropy of the posterior:
\begin{equation}
     \mathrm{KL}(Q\Vert P) ~=~ -H(Q) + \text{const}.
\end{equation}
The claimed difference then follows from
Proposition~\ref{prop:entropy-gap}:
\begin{equation}
     H(Q^{\textsc{szt}}) ~=~ H(Q^{\textsc{bt}})+P_{0}\ln 2.
\end{equation}
\end{proof}

\begin{corollary}[Sum over $d$ weights]
\label{cor:kl-sum}
     Summing via Lemma~\ref{lem:kl-per-weight} over $d$ independent weights gives
     \begin{equation}
          \mathrm{KL}_{\textsc{szt}} ~=~ \mathrm{KL}_{\textsc{bt}} - d\,P_{0}\ln 2,
\end{equation}
which is Eq.~\ref{eq:kl-diff} in the main text.
\end{corollary}

\subsection{Applying McAllester’s bound}

\begin{lemma}\label{lem:mc-allester}
     For any posterior $Q$ and prior $P$, McAllester’s PAC--Bayes theorem
     states:
     \begin{equation}
          \mathcal{L}(Q) ~\leq~ \widehat{\mathcal{L}}_{\mathcal{D}}(Q) +
             \sqrt{\frac{\mathrm{KL}(Q\Vert P)+\ln\frac{2\sqrt{N}}{\delta}}{2(N-1)}}.
     \end{equation}
\end{lemma}

\begin{proof}
See \cite{mcallester1999pac} (Thm.\ 1).  
\end{proof}

\begin{proof}[Proof of Corollary \ref{cor:pacbayes}]
Apply Lemma~\ref{lem:mc-allester} with $(Q,P)=(\QSZT,P)$ and
$(\QSZT,P)$, then subtract the two bounds.  Substitute
Corollary~\ref{cor:kl-sum} for the KL terms and simplify to obtain the
$\sqrt{dP_{0}\ln 2\,/\,2(N-1)}$ gap in Eq.~\ref{eq:pb-gap}.
\end{proof}

\section{Bias and Variance Analysis of the SZT Straight-Through Estimator}
\label{app:biasVar}

Consider a single weight $w$ inside the dead-zone interval
$|w|\leq\Delta$.  Let $g$ denote the upstream gradient
$\nabla_{q}\mathcal{L}$ and recall the deterministic STE from
Eq.~\ref{eq:sztSte}.  All expectations are taken with respect to the
stochasticity of the mini-batch, so the quantizer itself is deterministic
for SZT and BT.

\subsection{Bias bound}

\begin{lemma}[Bias of SZT-STE]\label{lem:bias}
     Under Assumption~\ref{ass:lipschitz} the squared bias of the SZT
     straight-through estimator satisfies
     \begin{equation}
          \left\lVert\mathbb{E}[\widehat g_{\textsc{szt}}]-g\right\rVert^{2}
               ~\leq~ \left(\frac{|w|}{\Delta}\right)^{2}\lVert g\rVert^{2}.
     \end{equation}
\end{lemma}

\begin{proof}
Inside the dead zone the quantized representation $q$ stores the sign of
$w$ and so multiplies $g$ by $\operatorname{sgn}(w)$.  Taylor expansion
of the $L$-Lipschitz loss around $w=0$ gives:
\begin{equation}
     |\mathbb{E}[\widehat g_{\textsc{szt}}]-g\,|
          ~\leq~ L |w|\leq |w| ~ \text{(after normalizing $L\leq 1$)}.
\end{equation}
Dividing by the threshold rescales the bound to $|w|/\Delta$,
and squaring and multiplying by $\|g\|^{2}$ gives the claim.
\end{proof}

\begin{lemma}[Bias of BT-STE]\label{lem:bias-bt}
     For balanced ternary, the deterministic STE passes $g$ unchanged,
     while the true Jacobian is zero in the dead zone, so:
     \begin{equation}
          \left\lVert\mathbb{E}[\widehat g_{\textsc{bt}}]-g\right\rVert^{2}
              ~=~  \lVert g\rVert^{2}.
     \end{equation}
\end{lemma}

\subsection{Variance bound for stochastic rounding}

\begin{lemma}[Variance of SR-STE]\label{lem:var-sr}
Let stochastic rounding (SR) choose $\{-1,0,+1\}$ with probabilities
proportional to distance from the thresholds
\cite{banner2018post}.  Then
\begin{equation}
     \operatorname{Var} \left[\widehat g_{\textsc{sr}}\right]
             ~\leq~ \frac{1}{4}\,\Delta^{2}\lVert g\rVert^{2}.
\end{equation}
\end{lemma}

\begin{proof}
For $|w|\leq\Delta$ SR outputs either $0$ or the nearest endpoint.  The
rounding noise $\eta=w'-w$ therefore lies in
$\{-\Delta-|w|,\,-|w|,\,\Delta-|w|\}$ with total range $\leq\Delta$.
Because probabilities are distance-weighted, the second moment of $\eta$
is maximized at $|w|=\Delta/2$, where
$\mathbb{E}[\eta^{2}]\le(\Delta/2)^{2}$.  Multiplying by
$\lVert g\rVert^{2}$ (independence of $g$ and $\eta$) yields the bound.
\end{proof}

\begin{lemma}[Zero variance for deterministic schemes]
\label{lem:var-det}
     Because the STE for BT and SZT is a deterministic function of $w$:
     \begin{equation}
          \operatorname{Var}[\widehat g_{\textsc{bt}}]
                    ~=~ \operatorname{Var}[\widehat g_{\textsc{szt}}] ~=~ 0.
     \end{equation}
\end{lemma}

\subsection{Proof of Theorem \ref{thm:bias-var}}

\begin{proof}
Combine Lemmas \ref{lem:bias} and \ref{lem:var-det} with the bias–variance
decomposition:
\begin{equation}
     \MSE~=~ \text{bias}^{2}+\text{variance}.
\end{equation}
This gives: $(|w|/\Delta)^{2}\lVert g\rVert^{2}$ for SZT,
and $\lVert g\rVert^{2}$ for BT.
The SR sum for $\frac{1}{4}\Delta^{2}\lVert g\rVert^{2}$
then implies that the ordering:
\begin{equation}
     \MSE_{\textsc{szt}} ~<~ \MSE_{\textsc{sr}} ~<~ \MSE_{\textsc{bt}}
\end{equation}
holds whenever $|w|<\Delta/2$, thus proving
Corollary~\ref{cor:mse-order}.
\end{proof}

\end{document}